\documentclass[final,onefignum,onetabnum]{siamart171218}

\usepackage{lipsum}
\usepackage{amsfonts}
\usepackage{graphicx}
\usepackage{epstopdf}
\usepackage{algorithmic}
\usepackage{color}
\usepackage{hhline}
\usepackage[normalem]{ulem}
\ifpdf
  \DeclareGraphicsExtensions{.eps,.pdf,.png,.jpg}
\else
  \DeclareGraphicsExtensions{.eps}
\fi
\usepackage{hyperref}
\usepackage{enumitem}
\setlist[enumerate]{leftmargin=.5in}
\setlist[itemize]{leftmargin=.5in}

\newcommand\T{\rule{0pt}{2.6ex}}       
\newcommand\B{\rule[-1.2ex]{0pt}{0pt}} 
\newsiamremark{remark}{Remark}
\newsiamremark{hypothesis}{Hypothesis}
\crefname{hypothesis}{Hypothesis}{Hypotheses}
\newsiamthm{claim}{Claim}

\headers{Spike-based primitives for graph algorithms}{Hamilton, Mintz and Schuman}

\title{Spike-based primitives for graph algorithms\thanks{Submitted on 25 MARCH 2019. \funding{This manuscript has been authored [in part] by UT-Battelle, LLC, under contract DE-AC05-00OR22725 with the US Department of Energy (DOE). The US government retains and the publisher, by accepting the article for publication, acknowledges that the US government retains a nonexclusive, paid-up, irrevocable, worldwide license to publish or reproduce the published form of this manuscript, or allow others to do so, for US government purposes. DOE will provide public access to these results of federally sponsored research in accordance with the DOE Public Access Plan (\url{http://energy.gov/downloads/doe-public-access-plan}).}}}

\author{Kathleen E. Hamilton\thanks{Computer Science and Engineering Division, Oak Ridge National Laboratory, Oak Ridge, TN (\textit{corresponding author} \email{hamiltonke@ornl.gov}).}
\and Tiffany M. Mintz \thanks{Computer Science and Mathematics Division, Oak Ridge National Laboratory, Oak Ridge, TN (\email{mintztm@ornl.gov}, \email{schumancd@ornl.gov}).}
\and Catherine D. Schuman \footnotemark[3]}

\makeatletter
\newcommand*{\addFileDependency}[1]{
  \typeout{(#1)}
  \@addtofilelist{#1}
  \IfFileExists{#1}{}{\typeout{No file #1.}}
}
\makeatother

\ifpdf
\hypersetup{
  pdftitle={Spiking graph algorithms},
  pdfauthor={K. E. Hamilton, T. M. Mintz, and C. D. Schuman}
}
\fi

\begin{document}

\maketitle

\begin{abstract}
In this paper we consider graph algorithms and graphical analysis as a new application for neuromorphic computing platforms. We demonstrate how the nonlinear dynamics of spiking neurons can be used to implement low-level graph operations.  Our results are hardware agnostic, and we present multiple versions of routines that can utilize static synapses or require synapse plasticity.
\end{abstract}



\section{Introduction}
\label{sec:introduction}
Spike-based algorithms are designed to run on neuromorphic hardware and rely on discrete electrical pulses transmitted via weighted connections to implement programmatic steps. In this work we identify and establish how many general graph algorithms can be implemented using spiking neurons for future deployment on neuromorphic hardware.  While computational models of spiking neurons have been developed to describe the behavior of biological neurons, our goal in this work is to fundamentally exploit the spiking behavior encoded by a given mathematical neuron model.  Developing new applications for neuromorphic hardware is a growing field of research \cite{paper24_pmes2018}, with algorithms being developed for community detection \cite{NCAMA_arXiv,Hamilton_PRE_arXiv,quiles2010label,quiles2013dynamical}, shortest path finding \cite{paper27_pmes2018}, Markovian random walks \cite{severa2018spiking} and general scientific computing \cite{severa2016spiking}.  

Specialized processors optimized to implemented spiking neurons and spike-based algorithms are known as spiking neuromorphic computing systems (SNCs) \cite{schuman2017survey}.  Many platforms are under development (e.g., IBM TrueNorth \cite{merolla2014million}, Intel Loihi \cite{davies2018loihi}, SpiNNaker \cite{painkras2013spinnaker}, non-volatile memory devices \cite{burr2017neuromorphic}) with different capabilities \cite{indiveri2015memory} but similar strengths such as low communication costs. 

This paper introduces mathematical proofs of spike-based primitives for graphical algorithms as well as examples of hardware agnostic implementations.  The routines in this paper are algorithms for either extraction, enumeration or verification.  Extraction algorithms are used to retrieve graph elements: nearest neighbors of a vertex, the neighborhood graph of a vertex, single source shortest path, and single source shortest cycle.  Enumeration algorithms count graphical elements (such as triangles) in a graph.  Verification algorithms are applied to a subset of neurons and return a Boolean variable (for example testing if a subgraph is a clique). Motivated in part by the GraphBLAS (Basic Linear Algebra Subprograms) library (which introduced sparse matrix-based primitives for graph algorithms) \cite{mattson2013standards,kepner2016mathematical,buluc2017design} we aim to stimulate research interest in the design of graph algorithms for event-based computation.

By definition spikes are sent from a neuron simultaneously along all outgoing synapses, and transmitting information through the isotropic firing patterns of spiking neurons is the true advantage of using event-based computation.  However fully leveraging this inherently parallel behavior may require the use of synapse plasticity.  As not all neuromorphic computing platforms can fully implement synapse plasticity, we present several examples of routines that use iterative execution (usually used when the implementation does not implement synaptic plasticity), or through parallelized algorithms (usually used when the implementation does include synaptic plasticity).  It is worth noting that the relative threshold values, synapse weights, delays and refractory periods will be dependent on individual hardware capabilities, but we focus on presenting our routines with general guidelines showing how spiking neurons can extract and implement each routine.

General definitions for graphs, spiking neurons, and complexity of event-driven computation are given in \Cref{sec:defs}.  Our mathematical results are organized into three sections (and summarized in \Cref{tab:methods_summary}): extraction in \Cref{sec:extraction}, enumeration in \Cref{sec:triangles} and verification in \Cref{sec:verification}. The maximum clock time (MCT) to execute each method listed in the third column is an upper bound on the number of clock steps needed to simulate the spiking neurons.  We use the ratio ($v_{th}/s_w$) of spike threshold to synaptic weight as an unique network characteristic and for routines that require multiple steps this value is listed for each step. Instantiating a set of spiking neurons, or reading the current neuron states can add significant computational overhead and are dependent on the choice of hardware.  For these steps we report the number of times the current state of a system needs to be measured (\textit{Reads}) and the number of times a system needs to be instantiated (\textit{Writes}).  
\begin{table}[htbp]
\caption{Summary of Extraction (Ex.), Enumeration (En.) and Verification (V.) routines}
\centering
\begin{tabular}{| l | c | c | c | c | c |}
\hline
Method & Synapse & MCT & $v_{th}/s_{w}$ & Reads & Writes \T\B\\
\hline
Nearest neighbor (Ex.) & Static & 1 & 1 & 0 & 1\\
Eccentricity (Ex.) & Static & $N$ & 1 & 0 & 1\\
Subgraph (Ex.) & Static & $N+1$ & 1 & 0 & 2\\
Subgraph (Ex.) & Plastic & 2 & 1 & 1 & 2\\
\hline
Triangle (En.) (edge) & Static & 1 & 2 & 0 & 1\\
Triangle (En.) (vertex) & Static & $d+1$ & 1,2 & 0 & $d+1$ \\
Triangle (En.) (clique)  & Static & $\binom{d}{2} + 1$ & 1,2 & 0 & $\binom{d}{2}+1$\B\\ 
\hline
Clique (V.) & Static & 1 & $(n-1)$ & 0 & 1 \\
Clique (V.) & Plastic & 1 & $(n-1)$ & 1 & 1\\
\hline \end{tabular}
\label{tab:methods_summary}
\end{table}

\section{Definitions}
\label{sec:defs}
In this section we introduce define relevant graph quantities, neuronal and synaptic parameters, and define how a graph is mapped onto spiking neurons.  

\subsection{Graph definitions}
An undirected graph $\mathcal{G}(V,E)$ is defined by a vertex set $V(\mathcal{G})$ and an edge set $E(\mathcal{G})$. A directed graph $\mathcal{D}(V,E)$ is defined by a vertex set $V(\mathcal{D})$ and a set of directed arcs $E(\mathcal{D})$.  The notation $E$ can refer to either a set of undirected edges or directed arcs, but directed edges are labeled by the arc direction $e_{i \rightarrow j} \neq e_{j \rightarrow i}$ while undirected edges are labeled by the two terminal nodes $e_{ij} = e_{ji}$.  We convert unweighted graphs into weighted graphs by assigning each $e \in E$ a length of 1.  To avoid confusion with synaptic weights, we refer to graph edges as having lengths rather than weights.  Finally, throughout this work we only consider graphs that are connected and do not contain self-loops.

\subsection{Spiking neuron system definitions}
We employ spiking neuron systems (SNS) $\mathcal{S}(n_{\mathbf{l}},s_{\mathbf{k}})$ similar to those described in \cite{schuman2017simulating}: they are are defined by a set of nonlinear, deterministic neurons $n_{\mathbf{l}}$ and a set of synapses $s_{\mathbf{k}}$.  An SNS is a connected graph that is defined by a neuron set $N_{\mathbf{l}}(\mathcal{S})$ and a synapse set $W_{\mathbf{k}}(\mathcal{S})$.  Here, we assume that the neurons are deterministic and the spiking behavior is defined by the internal state of each neuron, quantified by a time-dependent electrical potential.  

The neurons are parameterized by the spike threshold and the refractory period $\mathbf{l}=\lbrace v_{th},t_R\rbrace$  (respectively).  When a neuron fires, it can be classified as ``external'' or ``internal'' firing; external firing is caused by an external stimulus applied to a neuron,
internal firing is cause by multiple spikes arriving from other neurons.  The synaptic edges are parameterized by weight, delay, and plasticity $\mathbf{k} = \lbrace s_w, \delta, \alpha \rbrace$ (respectively). We parameterize plasticity by the learning rate $\alpha$.  Static synapses cannot learn and are identified with $\alpha=0$.  If $\alpha \neq 0$, then the synapses are plastic, but further detail about the learning rule must be specified.  We present algorithms throughout this paper that can be implemented with either static or plastic synapses, and require that plastic synapses can realize a form of spike timing dependent plasticity (STDP).  

For all of the algorithms presented in this paper, the SNS are constructed from a given graph ($\mathcal{G}$ or $\mathcal{D}$) via \textit{direct mapping}.  
\begin{definition}[Direct Mapping]\label{def:direct_mapping}
 A graph ($\mathcal{G}$ or $\mathcal{D}$) is directly mapped to an SNS by defining a neuron $n_i \in n$ for each $v_i \in V$  and a synapse or pair of synapses for each edge $e \in E$.  Directed arcs are mapped to directed synapses, and undirected edges are mapped to symmetric pairs of synapses $e_{ij} \to (s_{i\to j},s_{j\to i})$.
\end{definition}
We assume that the graphs can be embedded directly into an SNS such that there are sufficient neurons and synapses and connectivity between them to realize the graph in the SNS. In future work, we plan to investigate adapting these algorithms to neuromorphic systems that cannot realize the graph in an SNS due to connectivity restrictions. Direct mapping preserves the connectivity of the original graph, and in particular ensures that graph elements such as triangles and cycles are also present in the SNS.

\begin{definition}[Triangle]\label{def:triangle} A triangle on an undirected graph is defined by an unordered tuple of vertices $(v_i,v_j,v_k) \in V(\mathcal{G})$ where $(e_{ij},e_{jk},e_{ik}) \in E(\mathcal{G})$.  Under direct mapping, any triangle $(v_i,v_j,v_k)$ will map to a tuple of neurons $(n_i,n_j,n_k) \in \mathcal{S}(N,W)$.
\end{definition}

Information is extracted from an SNS by analyzing the spike behavior of individual neurons, by measuring the time dependent synaptic weights in the system, or both.  The overall spiking behavior is dependent on the neuronal and synaptic parameters in the SNS.  By changing these parameters we can implement different algorithms on the same set of neurons and synapses.  

\begin{figure}[htbp]
  \centering
  \includegraphics[width=0.5\textwidth]{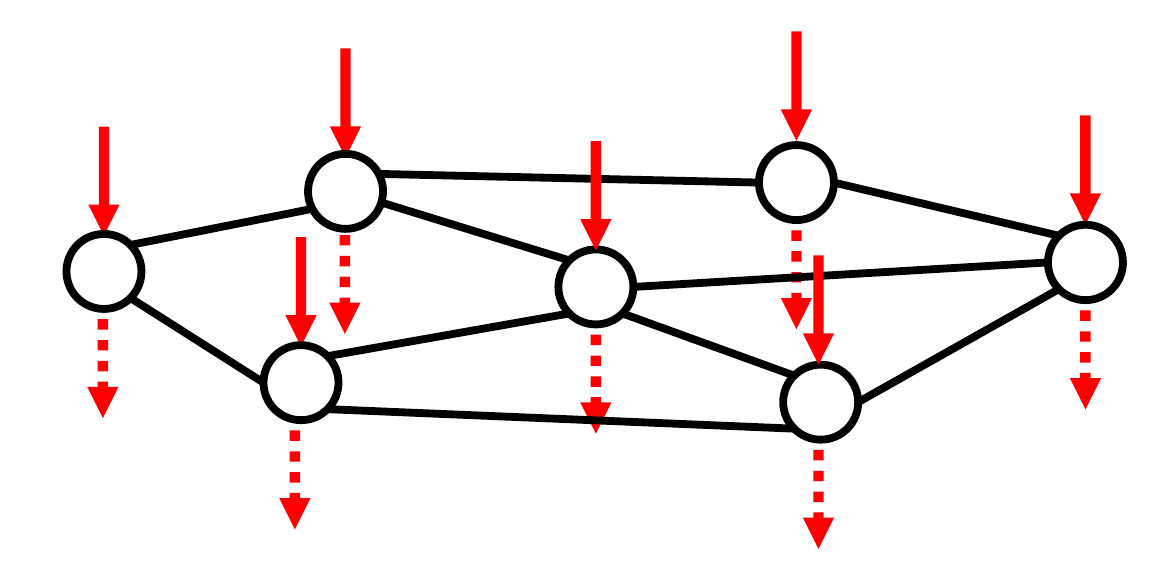}
  \caption{Schematic showing how external stimuli can be applied to neurons of an SNS. Discrete spikes are fed to a neuron via an external source along directed arcs (solid red arrow) with synaptic weights defined to be greater than the neuron firing threshold.  The arrival of an external spike at a neuron will immediately cause it to internally fire.}
  \label{fig:general_SNS_IO}
\end{figure}
For the graph algorithms described in this work, synaptic plasticity plays the role of an indicator function for quick identification of which neurons caused subsequent firing in the SNS.  A simple demonstration of STDP is shown in Fig. \ref{fig:simple_STDP}. Since STDP is dependent on the response of neurons to arrival spikes, multiple time steps are needed in order to fully observe the effects of spikes travelling through the system.
\begin{figure}[htbp]
  \centering
  \includegraphics[width=0.7\textwidth]{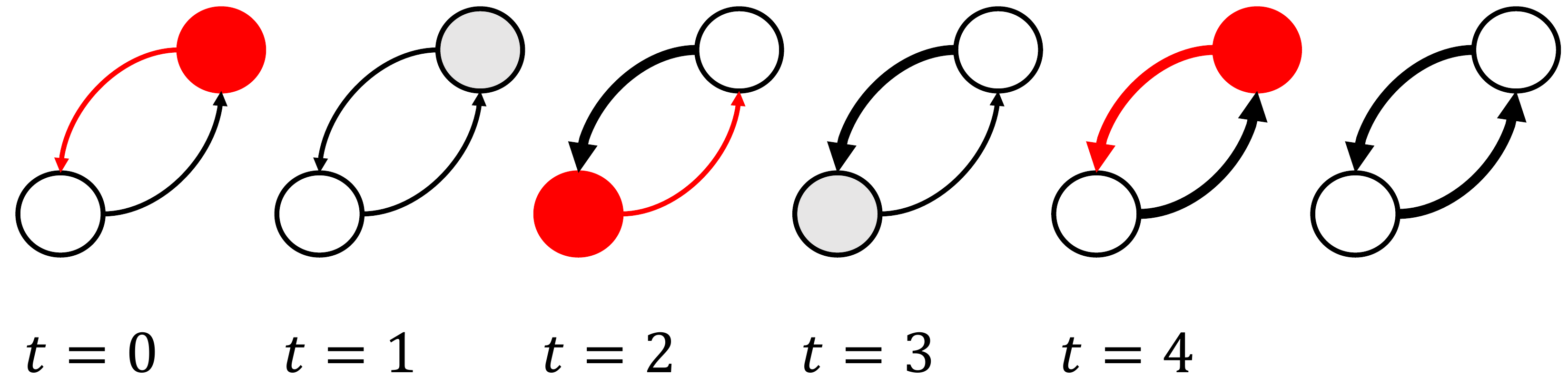}
  \caption{1-step STDP in $\mathcal{S}(N_{v_{th}=1,t_R=1},W_{s_w=1,\delta=2})$, active neurons and synapses are red, inactive neurons and synapses are outlined in black, and neurons in a refractory period are grey.  At $t=0$ the upper neuron (red) fires spikes along its outgoing synapses (red), then at $t=1$  enters a refractory period (grey). The delay $\delta=2$ prevents the spike from arriving at the bottom neuron until $t=2$, at which point the bottom neuron fires (red) and the synapse which carried the spike from the upper neuron is strengthened (shown by increased line width).  The process is repeated, starting from the bottom neuron for time steps $t=2,3,4$.  If the weights are read out at this point, both synapses have potentiated (increased in weight).}
  \label{fig:simple_STDP}
\end{figure} 

\subsection{Neuron driving}
\label{sec:building_blocks}
In addition to the connections defined from the graph mapping, each neuron has an external set of directed edges that allow for the application of stimuli (represented by a directed arc that terminates at a neuron) and a directed edge that allows for the recording of when a neuron fires a spike.  These directed connections are defined with positive synaptic weights that exceed the firing threshold of a neuron, ensuring that a spike that is directed along the input connection will drive a neuron to internally fire.  A cartoon showing the general input and output connections to an SNS is shown in \cref{fig:general_SNS_IO}.  We implement spike-based graph algorithms using the spiking dynamics generated by applying external stimuli to a single neuron (single neuron driving) or stimuli simultaneously to multiple neurons (multiple neuron driving).

To implement single neuron driving a SNS $\mathcal{S}$ is constructed from a graph $\mathcal{G}$, then external stimuli are applied to a specific neuron $n_i \in N(\mathcal{S})$.  Single neuron driving can also be used to extract the nearest neighbors of a neuron, and distance characteristics of a graph (e.g. vertex eccentricity, graph radius and diameter) (see \Cref{sec:extraction}).

For multiple neuron driving, if $m$ neurons are driven the maximum number of spikes that can arrive at a (non-driven) neuron within one time step is $m$.  If the SNS is defined as $\mathcal{S}(N_{v_{th}=1},W_{s_w=1,\alpha=0})$ then the use of multiple neuron driving can be used to identify within $1$ time step all neurons that are within $1$ edge of any of the neurons being driven, but the spike raster alone is insufficient to distinguish which neurons are closest to which driven neuron. By introducing plasticity and a simple spike timing dependent learning rule, then the nearest neighbors of each driven neuron can be identified from the synapses that have potentiated and have $s_w>1$.

When a neuron fires a spike, that spike is transmitted along all its synapses; this results in spikes simultaneously arriving at multiple target neurons.  Multiple neuron driving forms the building block of many parallelized algorithms, and can be implemented by sending external stimuli to a set of neurons, with no information about their interconnections, or we can send external stimuli to neurons which are known to be connected by a synapse.  Simply increasing the number of neurons receiving external stimuli is not guaranteed to result in gaining more information from a graph (see \Cref{fig:multiple_neuron_driving_IO}). By incorporating more complex neuron structures (heterogeneous spike thresholds, synapse plasticity) it is possible to design parallel implementations of algorithms or extract higher order correlated graph elements such as triangles.  For example, with single neuron driving the maximum number of spikes that can arrive at a (non-driven) neuron within one time step is $1$.  

\begin{figure}[htbp]
  \centering
  \includegraphics[width=0.7\textwidth]{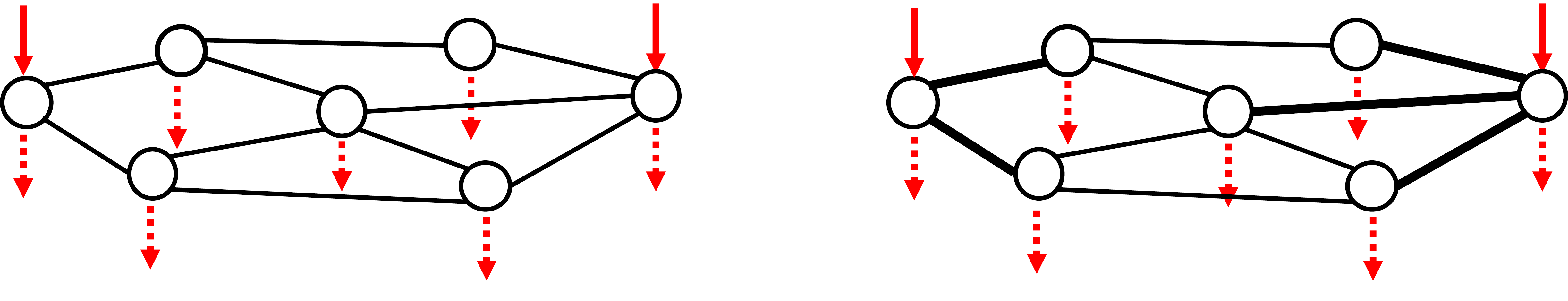}
  \caption{[Left] Schematic showing external stimulus applied to a multiple neurons in a system without synapse plasticity and neurons that internally fire within $1$ time step. [Right] Schematic showing neurons that internally fire within $1$ time step and how that affects individual synapse weights.}
  \label{fig:multiple_neuron_driving_IO}
\end{figure} 

\subsection{Computational complexity}
\label{sec:complexity_summary}
Computational complexity can be defined as the number of steps needed to execute a program.  Our approach to designing algorithmic primitives for spiking neurons relies on an expanded concept of complexity.  Rather than defining it in terms of the complexity of the input spikes or the time to execute a program, we define the computational complexity of a spike-based algorithm as the number of time steps needed to execute a program and the number of unique spiking neuron systems that have to be constructed and executed.  There are tasks for non-spiking computational workloads that can be identified as ``embarrassingly parallel,'' for which dividing a single task among $n$ threads can lead to significant speed up due to the fact that each individual partition can be executed with minimal computational costs.  A major advantage of spike-based algorithms is the isotropic firing pattern of individual neurons which leads to a degree of parallelization in all spiking algorithms.  In terms of neural algorithms, there are tasks that are strictly serial and those that are parallelized.  Throughout this work we refer to algorithms that rely on driving a single neuron during a time step as ``serial'' programs, while algorithms that utilize multiple neuron driving are referred to as ``parallel'' programs. Algorithms (e.g., subgraph extraction) that are ``embarrassingly parallel'' can be executed in significantly fewer time steps when they are implemented with multiple neuron rather than single neuron driving. 

There is an additional component in the complexity of these algorithms, which is writing and reading network information to and from the neuromorphic system.  ``Writing" the network corresponds to configuring the network topology and setting the appropriate parameters (synaptic weights and delays, neuronal thresholds) on the neuromoprhic system.  ``Reading" corresponds to reading out the updated network from the neuromorphic system after activity has occurred. Reading a network only makes sense if plasticity occurs in the neuromorphic system, because plasticity will change the network's parameters; if there is no plasticity, the network remained unchanged due to behavior. Because writing and reading networks is a non-trivial component of real execution time on neuromorphic and because the way it is implemented (and thus how much time it takes) is implementation-specific, for each of the algorithms we simply specify the number of network reads and writes required to execute that algorithm.

\section{Extraction algorithms}
\label{sec:extraction}
GraphBLAS routines are built from the adjacency and incidence matrices of a graph.  For example, the neighbors of vertex $v_i \in V(\mathcal{G})$ can be found using the equation $\overline{\mathbf{y}} = A_{\mathcal{G}}\overline{\mathbf{x}}$ where the adjacency matrix of the graph $A_{\mathcal{G}}$ is multiplied with a binary vector $\overline{\mathbf{x}}$ with $\mathbf{x}[i]=1$ the only non-zero value.  The resultant vector $\overline{\mathbf{y}}$ is a binary vector with non-zero entries $\mathbf{y}[j]=1$ where $j$ corresponds to the neighboring vertices.  

Rather than implement matrix multiplication with spiking neurons, we use the spiking behavior to transmit equivalent information.  For the example of nearest neighbor identification, this is done by defining a homogeneous SNS with static synapses $\mathcal{S}(N_{v_{th}=1},W_{s_w=1,\alpha=0})$ constructed from a graph $\mathcal{G}$ via direct mapping (\Cref{def:direct_mapping}). By applying an external stimulus to a single neuron (analogous to defining the vector $\overline{\mathbf{x}}$) we return a response from only the nearest neighbors (analogous to the vector $\overline{\mathbf{y}}$).

\begin{theorem}[Nearest neighbor extraction]\label{th:nearest_neighbors}
Directly map an undirected graph $\mathcal{G}$ to the SNS $\mathcal{S}(N_{v_{th}=1},W_{s_w=1,\alpha=0})$.  When an external stimuli is applied to a single neuron $n_i$ and allowed to propagate for a single time step, internal spiking can only occur in the nearest neighbors of $n_i$. 
\end{theorem}
\begin{proof}
The proof is done largely by inspection and relies on the isotropic firing patterns of spiking neurons.  Feeding external spikes to a single neuron ($n_i$) will cause ($n_i$) to fire spikes along all synapses that originate at $n_i$.  Within one time step, the only neurons that can internally fire are those that have received at least one spike, which is possible for neurons within $1$ edge of $n_i$.
\end{proof}
\begin{corollary}\label{cor:subgraph_edges_single}
If a neuron $n_j$ in $\mathcal{S}(N_{v_{th}=1},W_{s_w=1,\alpha=0})$ internally fires within one time step after $n_i$ externally fires, then the edge $e_{ij}$ must exist on the graph.
\end{corollary}

If the amount of time that spikes are allowed to propagate through $\mathcal{S}$ is increased, then single neuron driving can also be used to identify all neurons reachable by at least a finite number of non-backtracking steps from the driven neuron.  This non-backtracking condition is strongly enforced by imposing a large refractory period on each neuron.  
\begin{corollary}\label{cor:spike_ball}
If an SNS $\mathcal{S}(N_{v_{th}=1,t_R=|N(\mathcal{S})|},W_{s_w=1,\alpha=0})$ is defined from a graph $\mathcal{G}$, and single neuron driving applied to $n_i \in N(\mathcal{S})$, then after $r$ time steps to then all neurons that can be reached with at least $r$ non-backtracking steps will have fired at least once.
\end{corollary}
\begin{proof}
Define time step $t=0$ to be the time at which $n_i$ will fire a spike due to the external stimulus and then enter its refractory period.  At time step $t=1$, all nearest neighbors $\lbrace n^{\prime} \rbrace$ of $n_i$ will internally fire, however all spikes returning to $n_i$ will not induce further internal firing. At time step $t>1$ only neurons that have not internally fired in a previous time step will fire.  At time step $t=r$, if a neuron $n_j$ fires then it is connected to $n_i$ by a $r$-length non-backtracking path.  
\end{proof}
One result of \Cref{cor:spike_ball} is that spiking dynamics can be used to define an upper bound on the shortest path length between two vertices $(v_i,v_j)\in V(\mathcal{G})$.  If single neuron driving is applied to $n_i$, and $n_j$ fires within $r$ time steps then the shortest path connecting $|\mathcal{P}(v_i \to v_j)| \leq r$.  

A second result of \Cref{cor:spike_ball} is that spiking dynamics can be used to measure the eccentricity of a vertex $\epsilon(v_i) \in V(\mathcal{G})$.  The eccentricity of a vertex $\epsilon(v_i)$ is the longest shortest path $\mathcal{P}(v_i \to v_j)$ for any $v_j \in V(\mathcal{G})$ \cite{harary1969graph}.  The large refractory period $t_R=|N(\mathcal{S})|$ is defined to enforce the non-backtracking condition, but also ensures that the spike firing in $\mathcal{S}$ will terminate once every neuron has fired one spike. If single neuron driving is applied to a neuron $n_i$, and spikes are allowed to propagate until the firing pattern terminates at time step $R$, then the eccentricity $\epsilon(v_i) = R$.  For the case of $\mathcal{G}=K_N$, then $\epsilon(v_i)=1$ and the spiking in $\mathcal{S}$ will terminate after $1$ time step (MCT=1). For the case of $v_i$ being a terminal node of $\mathcal{G}=\mathcal{P}_N$, then the spiking in $\mathcal{S}$ will terminate after $N$ time steps (MCT=N).

Single neuron driving of $n_i$ can be used to identify nearest neighbors, or find an upper bound on the shortest path between two vertices, or define the eccentricity $\epsilon(v_i)$.  However with static synapses it will not identify the unique edges that exist in the neighborhood of $n_i$, or comprise $\mathcal{P}(v_i \to v_j)$.  These questions are related to subgraph extraction, which is described in the following section.

\subsection{General subgraph extraction}
\label{sec:subgraph}
In GraphBLAS, a subgraph is obtained from a larger graph by extracting columns and rows from the original graph adjacency matrix to form an adjacency matrix on a vertex subset.  We can extract a subgraph $\mathcal{G}^{\prime} \subseteq \mathcal{G}$ defined on a vertex subset $V^{\prime}$ using a two-step process.  The vertices of the subgraph are first identified, as the neuron subset $N^{\prime} \subseteq N(\mathcal{S})$.  All edges that exist on the subgraph can be identified from the spiking dynamics generated with single neuron driving on each $n_i \in \lbrace n^{\prime} \rbrace$ or by using multiple neuron driving for the entire subset $\lbrace n^{\prime} \rbrace$.

It cannot be assumed that all neighbors of a neuron are contained in the subgraph, and as a result external driving of a neuron must result in only a subset of nearest neighbor neurons internally firing.  This is achieved by first instantiating an SNS where the synapses are too weak to induce internal firing with a single spike (e.g. $v_{th}=|E(\mathcal{G})|, s_w=1.0$). Then, for the neurons in the subset $\lbrace n^{\prime} \rbrace$, the spike thresholds of these neurons are lowered to $v_{th}=s_w=1$.  Driving any neuron in the subset $N^{\prime}$ will cause spikes to be sent to all neighbors, but only those with lowered thresholds will internally fire.  The edges which should be added to the subgraph can be identified with single neuron driving and iterating over the neuron set $\lbrace n^{\prime} \rbrace$,  or they can be identified from synaptic weights when performing parallel subgraph extraction.

\begin{theorem}[Iterative subgraph extraction]\label{thm:subgraph_neurons_static}
If an SNS with static synapses is defined by directly mapping $\mathcal{G}$ to $\mathcal{S}(N,W)$, then it is possible to set the neuron parameters $l^{\prime}$ for $n_i \in \lbrace n^{\prime} \rbrace$ such that applying a stimulus to any neuron in the neuron subset $\lbrace n^{\prime} \rbrace$ will only generate a spike response from the remaining neurons in $\lbrace n^{\prime} \rbrace$.
\end{theorem}
\begin{proof} Instantiate $\mathcal{S}(N_{v_{th}=2},W_{s_w=1,\alpha=0})$ an SNS with spike threshold for all neurons $v_{th}>s_w$, if the spike threshold for a subset of neurons $N^{\prime}$ is reduced such that $v_{th} = s_w$ then applying single neuron driving for $t=1$ time steps to $n_i \in \lbrace n^{\prime} \rbrace$ will only induce internal firing in the remaining neurons of $\lbrace n^{\prime} \rbrace \setminus n_i$. 
\end{proof}
Using the result from \cref{cor:subgraph_edges_single} the edges of a subgraph can be extracted from a larger graph through iterative application of single neuron driving.  

A parallelized implementation of subgraph extraction requires multiple neuron driving and synaptic plasticity.  
\begin{theorem}[Parallel subgraph extraction]\label{cor:subgraph_edges_multi}
If an undirected graph $\mathcal{G}$ is directly mapped to an SNS $\mathcal{S}(N_{v_{th}=|E|},W_{s_w=1,\delta=1,\alpha=0.5})$ with synaptic learning defined by one-step STDP, then the neuron subset $n^{\prime}$ which defines a subgraph of interest can be extracted by reducing the firing threshold for $n_i \in n^{\prime}$ to $v_{th}=1$.  After $2$ time steps the edges of a subgraph are the synapses that have potentiated $s_w^{(1)}>s_w^{(0)}$.
\end{theorem}
\begin{proof}
The heterogeneous firing thresholds ensure that multiple neuron driving applied to the entire set $n^{\prime}$ will only cause internal firing in the neurons $n_i \in n^{\prime}$.  To ensure that all edges can be identified by potentiation, a nonzero synaptic delay is needed (as shown in \Cref{fig:simple_STDP}) and the edges of the subgraph are identified with one-step STDP \cite{gerstner2002spiking},
\begin{equation}
s_w^{t+1}=s_w^{t}+\alpha \eta_j\eta_i.
\end{equation}
\end{proof}
In the following example, the neighborhood graph is extracted with spiking neurons.  This example has been implemented on a simulated memristive platform in Ref. \cite{paper27_pmes2018}.
\subsection*{Example: neighborhood extraction}
\label{sec:neighborhood_graph}
A neighborhood of vertex $v_s$ in a graph $G$, $N_G(v_s)$, is defined as the subgraph made up of all vertices $v_j \in V(G)$ such that $e_{i,j} \in E(G)$ and the set of all edges that connect the vertices in this set.  To extract $N_G(v_s)$, we configure the corresponding SNS with the following parameters: $(v_{th}=1, t_R=1)$ for all neurons, and $(\delta = 2, s_w = 1)$ for all synapses.  

First the vertices of $N_G(v_s)$ are identified, then the edges are identified.  To identify the vertices, we apply single neuron driving to $n_s$ at time step 0 and simulate the network for exactly two time steps (due to the delay $\delta=2$).  The neurons that fire in those two time steps will be $\lbrace n^{\prime} \cup n_s \rbrace$, the set of neurons that are adjacent to $n_s$ and $n_s$, which correspond directly to the vertices in our desired subgraph.  To identify the edges, we increase the spike threshold for $\lbrace n \rbrace \setminus \lbrace n^{\prime} \cup n_s \rbrace$ to $|E(G)|+1$.  We reset the simulation, load the new graph with updated thresholds and $s_w=1$, and simultaneously stimulate the neuron set $\lbrace n^{\prime} \cup n_s \rbrace$.  After two time steps we read all out all synapses that have $s_w > 1$ (the original weight) to obtain the edges that have potentiated and thus are edges of $N_G(v_s)$. 

\section{Enumeration algorithms}
\label{sec:triangles}
In this section we derive methods for counting triangles in subgraphs, based on Cohen's triangle enumeration algorithm for undirected graphs \cite{cohen2009graph}.  First, we describe how to count all triangles that contain a given edge $e_{ij}$.  Then, we show how to count the number of triangles that contain a given vertex $v_i$.  These methods rely on single and multiple neuron driving and can all be executed using only static synapses.  

The SNS in this section have higher spike thresholds than in \Cref{sec:extraction} and unit synaptic weight.  With multiple neuron driving and higher spike thresholds, it is possible to identify correlated vertexes on the graph (triangles) as shown in Fig. \ref{fig:multiple_neuron_driving_triangle_IO}.  

\begin{figure}[htbp]
  \centering
  \includegraphics[width=0.7\textwidth]{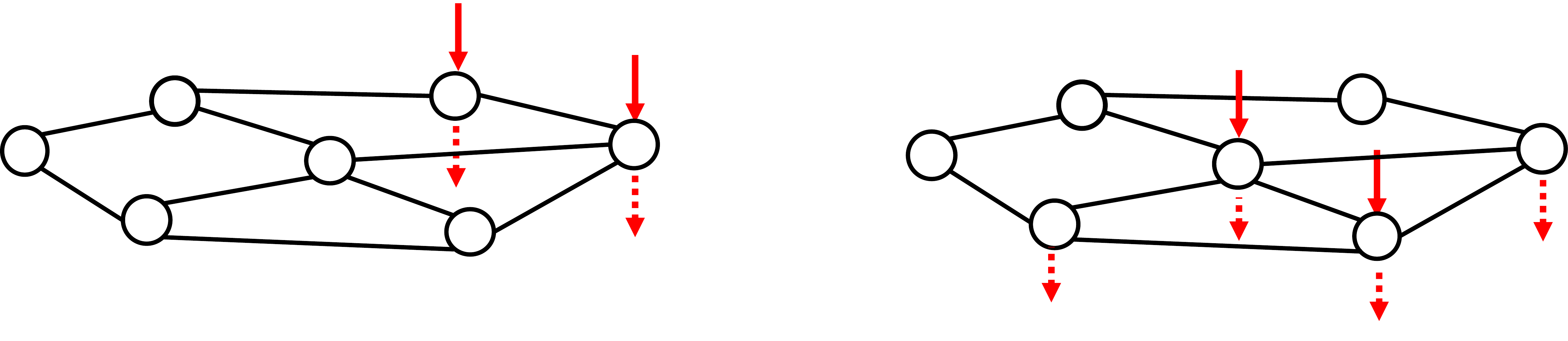}
  \caption{For $\mathcal{S}(N_{v_{th}=2},W_{s_w=1})$ [Left] External stimuli applied to two neurons connected by an edge which is not contained in any triangles on the graph does not generate any internal firing in $\mathcal{S}$. [Right] External stimuli applied to two neurons connected by an edge which is contained in two triangles on the graph can cause internal firing in $\mathcal{S}$.}
  \label{fig:multiple_neuron_driving_triangle_IO}
\end{figure} 

On a given graph, the neighborhood of a single vertex $N(v_0)$ is the induced subgraph on $v_0 \cup \lbrace v^{\prime} \rbrace$, where $\lbrace v^{\prime} \rbrace$ are all vertices that are connected to $v_0$ by an edge.  We extend this concept to a localized, correlated region containing a single edge and define a ``triangle neighborhood'' as follows:
\begin{definition}[Triangle Neighborhood]\label{def:triangle_neighborhood}
The triangle neighborhood $N_{\bigtriangleup}(e_{ij})$ is the induced subgraph of $\mathcal{G}$ on the vertices $v_i,v_j$ and all vertices $\lbrace v_k \rbrace$ that form triangles $(i,j,k)$ with the edge $e_{ij}$.
\end{definition}

From \cref{def:triangle_neighborhood,def:triangle}, and the nearest neighbor finding of \Cref{th:nearest_neighbors}, triangle enumeration is defined in two ways: edge-wise or vertex-wise.  To count all triangles that contain a given edge $e_{ij}$, multiple neuron driving is applied to the neurons $n_i, n_j$ and triangles ($n_i,n_j,n_k$) are identified by the neurons $n_k$ that fire. To count all triangles that contain a given vertex $v_i$ single neuron driving is applied to the neuron $n_i$ to identify all open wedges. With multiple neuron driving we can identify which wedges are closed (i.e., triangles) by choosing two edges $e_{ij},e_{ik}$ determining if $e_{jk} \in E(\mathcal{G})$.  An alternative vertex-wise triangle enumeration is presented which is similar to the clique verification methods descried in \Cref{sec:verification}.

\begin{theorem}[Triangle Enumeration for a single edge]
\label{thm:single_edge_triangle_count}
Construct an SNS $\mathcal{S}(N_{v_{th}=2},W_{s_w=1,\alpha=0})$ via a direct mapping from an undirected graph $\mathcal{G}$ and apply an external stimuli to two neurons $(n_i,n_j)$ which are connected by a synapse $s_{ij}$. Within one time step internal spiking can only occur in the neurons $\lbrace n_k \rbrace$ that form triangles $(n_i,n_j,n_k)$. 
\end{theorem}
\begin{proof}
With the restriction that any internal firing must occur within $1$ time step, for any neuron to internally fire then $2$ spikes must arrive from the pair of externally driven neurons $n_i, n_j$.  In order for a neuron $n_k$ to fire it must be connected to both neurons that are being externally driven.  It is known that the edge $(i,j)$ exists on the graph and using the result from \Cref{cor:subgraph_edges_single} once $n_k$ internally fires this implies that edges $(i,k)$ and $(j,k)$ exist on the graph and thus triangle $(i,j,k)$ must exist on the original graph.
\end{proof}

In the following two examples we demonstrate vertex-wise enumeration of triangles for a single neuron $n_i$ (directly mapped from a vertex $v_i$ of degree $d_i$) through iterative applications of \Cref{thm:single_edge_triangle_count}, or by utilizing local clique verification.

\subsubsection*{Example: Iterative enumeration of triangles for a single vertex}
\label{ex:triangle_vertex_static}
Similar to the subgraph extraction example \Cref{sec:neighborhood_graph}, enumerating the triangles that contain a specific vertex ($v_i$) requires two steps.  First, single neuron driving is applied to the corresponding neuron ($n_i$) and after one time step all nearest neighbors can be identified $\lbrace n^{\prime} \rbrace$.  Thus any triangle that contains $n_i$ must be completed by two neurons $(n_j, n_k) \in \lbrace n^{\prime} \rbrace$.  Since it is known that the edges $w_{ij},w_{ik}$ are contained in the original graph the final edge of a tuple $w_{jk}$ can be identified using the edge-based triangle enumeration described in \ref{thm:single_edge_triangle_count}) along the edge $w_{ij}$ and any neuron $n_k$ that fires must be part of a tuple with $(n_i, n_j)$. This can be repeated for all edges $w_{ii^{\prime}}$ connecting $(n_i,n_i^{\prime}) \forall n_i^{\prime}\in \lbrace n^{\prime} \rbrace$. 

The number of steps needed to enumerate all triangles connected to a single vertex will depend on the edge density of the neighborhood $N_G(v_i)$.  We can establish a lower bound for the case of $N_G(v_i)=\mathcal{T}(v_i)$ the neighborhood graph being a tree.  In this case the number of steps needed to verify that $0$ triangles exist is $d-1$.  

This method will identify triangles but will not avoid over counting.  If only the number of triangles are needed, then after running the edge-based enumeration on all $d_i$ edges connected to $v_i$, the total number of triangles counted needs to be divided by 2.  If the explicit tuples are needed, then storing each tuple as a sorted list $(n_i,n_j,n_k)$ as they are identified can be used to avoid over counting. Alternatively the triangles can be explicitly counted using a modified clique verification test.  Verification tests will be discussed in detail in \Cref{sec:verification}, here we introduce the specific case for triangles.

\subsubsection*{Example: Explicit enumeration of triangles for a single vertex}
\label{ex:triangles_vertex_clique}
If a graph on $3$ vertices is a triangle then it is also a clique ($K_3$).  In addition to the edge-wise and vertex-wise tests above we can also explicitly count all triangles that contain a vertex ($v_i$).  Once the nearest neighbors $\lbrace n^{\prime} \rbrace$ of the neuron $n_i$ with $d_i$ outgoing synapses are identified, the spike threshold for these neurons is reset to $v_{th}=2$ and the threshold for all other neurons in $\mathcal{S}$ is set to $|E|+1$.  From the set of $d_i$ neighbors, form all $ \binom{d_i}{2}$ pairs $(n_j,n_k)$ and use multiple neuron driving and \Cref{thm:clique_test} to identify valid triangles $(n_i, n_j, n_k)$.

\section{Verification algorithms}
\label{sec:verification}
In this section we describe several verification routines for clique identification.  These tests are applied to subgraphs $\mathcal{G}_n \subseteq \mathcal{G}$ of order $n$ to verify if $\mathcal{G}_n=K_n$.  First, we give a corollary to the edge-based triangle enumeration method in \Cref{thm:single_edge_triangle_count}. In \Cref{thm:single_edge_triangle_count}, we used multiple neuron driving across a single edge ($K_2$) combined with higher spike thresholds $v_{th}=2$ to identify $K_3$ cliques (triangles).  This process can be generalized and used to identify $K_n$ cliques that can be expanded into $K_{n+1}$ cliques. 
\begin{corollary}[Expanding $K_n$ to $K_{n+1}$]
\label{cor:clique_expansion}
Given a subgraph $\mathcal{G}_n$ on $n$ vertices known to form the clique $K_n$, any vertex that can be added to the subgraph to form $K_{n+1}$ can be identified using an SNS $\mathcal{S}(N_{t_R=0},W_{s_w=1,\delta=n})$ with heterogeneous firing thresholds. After applying multiple neuron driving to all neurons in $K_n$, any neuron outside $K_n$ that fires after $1$ time step can be added to $\mathcal{G}_n$ to form $K_{n+1}$. 
\end{corollary}
\begin{proof}
The neurons $N(\mathcal{S})$ are separated into disjoint subsets $\lbrace n \rbrace$, $\lbrace n^{\prime} \rbrace$. The set $\lbrace n \rbrace$ are the $K_n$ neurons and the firing thresholds are set to $v_{th}=n-1$. All neurons in $\lbrace n^{\prime} \rbrace = \lbrace N \rbrace \setminus \lbrace n \rbrace$ have firing thresholds set to $v_{th}=n$.

Applying multiple neuron driving to all neurons in $\lbrace n \rbrace$ will cause these $n$ neurons to externally fire.  If another vertex exists in the graph $\mathcal{G}$ that is connected to all vertices of $K_n$, then that corresponding neuron with $v_{th}=n$ will receive $n$ spikes and will internally fire.  
\end{proof}

The triangle enumeration of \Cref{thm:single_edge_triangle_count} can be implemented with either static or plastic synapses, and the same is true for \Cref{cor:clique_expansion}. Implementing the test in \Cref{cor:clique_expansion} will identify a set of $m$ vertices, any of which can be added to form $K_{n+1}$ from $K_n$.  We now define a general test that can be applied to a subgraph that will return True if the subgraph is $K_{n}$.  This test was previously used in \Cref{ex:triangles_vertex_clique}.

\begin{theorem}[Verifying a subgraph is a clique]
\label{thm:clique_test}
A subgraph $\mathcal{G}_n \subseteq \mathcal{G}$ of order $n$ is a clique $ \mathcal{G}_n = K_n$ if $\mathcal{G}_n$ can be directly mapped to a set of $n$ neurons with $v_{th}=n-1$ that will all internally fire when multiple neuron driving is applied to all neurons in the set.
\end{theorem}
\begin{proof}
Begin by directly mapping $\mathcal{G}$ to $\mathcal{S}(N_{t_R=1},W_{s_w=1,\delta=2})$ where the neurons have heterogeneous firing thresholds.  The vertices of $\mathcal{G}_n$ are mapped to the neuron set $\lbrace n^{\prime} \rbrace$ and $v_{th}=n-1$ for all $n_i \in \lbrace n^{\prime} \rbrace$.  To prevent any other neuron in $\mathcal{S}$ from firing, $v_{th}=|E|+1$ for all remaining neurons $n_j \in \lbrace N \rbrace \setminus \lbrace n^{\prime} \rbrace$.  

Multiple neuron driving is applied to all neurons in $\lbrace n^{\prime} \rbrace$ and each will fire $1$ spike along all outgoing synapses.  If the neurons form a clique $K_n$, then in the following time step $n-1$ spikes will arrive at each neuron.  After the delay, the $n-1$ spikes are sufficient to cause the $n$ neurons to fire a second time.  If a neuron $n_i \in \lbrace n \rbrace$ has degree $d_i < n-1$ then only $d_i$ spikes will arrive and this neuron will not reach the spike threshold $v_{th}=n-1$.  Thus if $\mathcal{G}_n \neq K_n$ only $n^{\prime\prime}<n^{\prime}$ neurons will fire.  
\end{proof}

The use of static or plastic synapses does not alter the overall complexity or simulation time of this test.  If plastic synapses are used then the active edges in the neuron set $\lbrace n \rbrace$ can be identified, and specifically for the case of $\mathcal{G}_n \neq K_n$ the absent edges can be found.  


\section{Conclusions}
\label{sec:conclusions}
The use of spiking neurons to perform tasks beyond neuroscience simulations and neural network style computation is a developing area of scientific computing.  Characteristics of spiking neurons, such as the isotropic firing of neurons that lends itself to massive parallelization, make them very attractive candidates for graph analysis.  We believe that by replacing matrix-based computing with spike-based computing there is the potential to speed up elements of graph analysis and in this work, we have proposed several ways in which spiking neurons can be incorporated into algorithmic primitives for graphs.  Since it it likely that neuromorphic systems will be included in future heterogeneous computers, expanding their use cases for non-neural network algorithms will allow more effective utilization of these processors in the future.  

Of all the neuromorphic hardware characteristics (built to efficiently execute spiking neuron models) the ability to implement synaptic plasticity poses the greatest advantage for graph algorithms. Because not all neuromorphic hardware systems implement synaptic plasticity, We have demonstrated multiple versions of algorithms that can utilize iterative solving (without synaptic plasticity) or parallel solving (with synaptic plasticity).  As such, though synaptic plasticity may be more difficult and energy intensive to implement in a neuromorphic system, it can reduce the computation time required for some of the tasks described here.

We plan to continue our development of graph algorithms for neuromorphic systems by incorporating these spike-based primitives into larger routines.  For example, hybrid algorithms can be constructed by combining these primitives with non-spiking data structures to store temporary variables.  In particular, we plan to investigate how one can implement distributed, parallel computing with neuromorphic systems in order to solve large-scale graph problems.

\section*{Acknowledgments}
Research sponsored in part by the Laboratory Directed Research and Development Program of Oak Ridge National Laboratory, managed by UT-Battelle, LLC, for the U. S. Department of Energy. This material is based upon work supported by the U.S. Department of Energy, Office of Science, Office of Advanced Scientific Computing Research, under contract number DE-AC05-00OR22725.

\bibliographystyle{unsrt}
\bibliography{bib_files/SGA_lit.bib}

\end{document}